\documentclass[numsec,webpdf,modern,medium,namedate]{oup-authoring-template}

\onecolumn

\graphicspath{{Fig/}}

\usepackage{natbib}
\usepackage{amsmath}
\usepackage{graphics}
\usepackage{subfig}

\usepackage{float}
\usepackage{caption}
\usepackage{fullpage}

\usepackage{newtxtext}
\usepackage[subscriptcorrection]{newtxmath}
\usepackage{epstopdf}

\theoremstyle{thmstyleone}%
\newtheorem{theorem}{Theorem}
\newtheorem{lemma}{Lemma}
\theoremstyle{thmstyletwo}%
\newtheorem{example}{Example}%
\theoremstyle{thmstylethree}%

\usepackage[doublespacing]{setspace}
\usepackage[fontsize=12pt]{fontsize}

\begin{document}

\journaltitle{Journal XXXX}
\DOI{XXXX}
\copyrightyear{XXXX}
\pubyear{XXXX}
\access{Advance Access Publication Date: Day Month Year}
\appnotes{Original article}

\firstpage{1}


\title[Unfaithful Probability Distribution in Causal DAG]{Unfaithful Probability Distributions in Binary Triple of Causality Directed Acyclic Graph}

\author[1,$\ast$]{Jingwei Liu \ORCID{0009-0003-1407-1969}}

\authormark{Jingwei Liu}

\address[1]{\orgdiv{School of Mathematical Sciences}, \orgname{Beihang University},
\orgaddress{ \postcode{100191}, \state{Beijing}, \country{P.R China}}}

\corresp[$\ast$]{Address for correspondence.Jingwei Liu,School of Mathematical Sciences, Beihang University, Beijing, 100191, P.R China. \href{email:liujingwei03@tsinghua.org.cn}{liujingwei03@tsinghua.org.cn}}

\received{Date}{0}{Year}
\revised{Date}{0}{Year}
\accepted{Date}{0}{Year}



\abstract{Faithfulness is the foundation of probability distribution and graph in causal discovery and causal inference. In this paper, several unfaithful probability distribution examples are constructed in three--vertices binary causality directed acyclic graph (DAG) structure, which are not faithful to causal DAGs described in J.M.,Robins,et al. Uniform consistency in causal inference. Biometrika (2003),90(3): 491--515. And the general unfaithful probability distribution with multiple independence and conditional independence in binary triple causal DAG is given.}


\maketitle

\section{Introduction}
Causal discovery has achieved great progress by representing causal relationships through directed acyclic graphs (DAG) with d--separation
and probability independence and conditional independence ( \citet{Pearl:1988}, \citet{Pearl:1995}, \citet{Pearl:2000}, \citet{Spirtes.etal:1993},
\citet{Spirtes.etal:2000},\citet{Koller.etal:2009}, \citet{Hernan.etal:2020}). Especially in case of vertices number  $n\geq 4$ for Gaussian
distribution and for discrete and binary distribution, conditional dependence relations have been established fruitfully (  \citet{Matus.etal:1995}, \citet{Matus:1995}, \citet{Matus:1999}, \citet{Studeny:1989},  \citet{Studeny:1992}, \citet{Studeny:2021},  \citet{Wermuth.etal:2009}).

Faithfulness of the probability distribution and the graph is originally developed in \citet{Spirtes.etal:1993}. \citet{Robins.etal:2003} addresses
the uniform consistency in causal statistical inference under causal faithful framework.Then, faithfulness is discussed and weak assumptions are proposed by \citet{Ramsey.etal:2006},\citet{Spirtes.etal:2014},\citet{Sadeghi:2017} and \citet{Marx.etal:2021}. The graph involved in faithfulness is called the perfect map of the probability distribution in \citet{Pearl:1988}, and perfact map is rare (\citet{Almond.etal:2015}).
A precise conclusion is that the set of unfaithful distributions has Lebesgue measure zero when the distribution of graph can be parameterised by a family of a finite dimensional parameter distribution (\citet{Meek:1995}, \citet{Spirtes.etal:2000}, \citet{Robins.etal:2003}).
However, faithfulness is still an important problem of probability and graph in causal discovery to use graphical methods for causal inference (\citet{Sadeghi:2017}).

For a triple  of three random variables $(X,Y,Z)$, there are 11 representative DAGs in literature. \citet{Robins.etal:2003} sums up 8 causal DAGs (see Fig 1) according to whether the DAG is faithful to $X\perp\!\!\!\perp Y$, while $Z$ is a confounder, or not; \citet{Sadeghi:2017} discusses unfaithfulness in two DAGs (see Fig 2 (ix) (x)), however a fourth variable is introduced. Fig 1 (iii) and Fig 2 (xi) are colliders in  Pearl (1988).  Fig 1(i), Fig 2 (ix) and Fig 2 (xi) are the three important conditional independence structures in causal analysis, namely diverging,serial,converging or collider (\citet{Koller.etal:2009}). In \citet{Ramsey.etal:2006}, violating faithfulness examples are constructed for showing how the PC algorithm errs (\citet{Spirtes.etal:1993},\citet{Spirtes.etal:2000}).

In order to intensively clarify the basic causal structure of triple $(X,Y,Z)$, several unfaithful probability distribution examples are constructed within the binary three-vertices causal DAG structure, which are not faithful to the 11 specific DAGs aforementioned, distinctly including DAGs proposed in \citet{Robins.etal:2003}. Theoretically, unfaithfulness of probability and graph is an unavoidable problem in S2 of PC algorithm (\citet{Spirtes.etal:2000}, \citet{Ramsey.etal:2006}) to discover a reasonable causal DAG.
The aim of the paper is to construct concise and direct examples to show unfaithful probability distribution,
by constructing compounding multiple independence and conditional independence simultaneously in the same probability.

In \S 2 the binary triple causal DAG and basic lemmas are reviewed, and two fundamental theorems for causal analysis are given. In \S 3 several unfaithful probability examples are constructed. In \S 4 the general unfaithful probability family of binary triple is obtained. And, the discussion is given in  \S 5.

\section{The binary three vertices causal DAG}\label{sec2}

Let $V=(X_1, . . . , X_n)$ be a vector of $n$--dimension random variables, $P$ be a joint probability distribution for $V$,  $G$ be a DAG with  $n$--vertices or nodes of $V$, $E$ be directed arrows set between some pairs of vertices, and there is no directed cycle in $G=(V,E)$.
The random variable $X_i$ takes values in $\mathcal{X}_i$, and arrow from $X_i$ to $X_j$ represent direct causal relationship from $X_i$ to $X_j$, where $X_i$ calls a ``parent'' of $X_j$, and $X_j$ is called a ``descendant'' of $X_i$ .
The triple $(G,V,P)$ is called the Spirtes--Glymour--Scheines causal model.  (\citet{Pearl:2000}, \citet{Spirtes.etal:2000}, \citet{Robins.etal:2003})

Let $pa_G(x_i)$ denote the set of ``parents'' of a variable $X_i$. If  distribution $P$ satisfies
\begin{equation}
P(x_1, . . . , x_k)= \prod_{i=1}^n P\{x_i |pa_G(x_i)\},
\end{equation}
where $P\{x_i |pa_G(xi)\}=P(x_i)$ when $pa_G(x_i)=\emptyset $. $P$ is called ``Markov compatibility'' or ``Markov factorisation'' to $G$ .

Let $\mathscr{P}(G)$ represent all distributions Markov to $G$.
Let $\mathscr{I}(P)$ denote all probability independence and conditional independence of the variables $V$ under $P$.
Let $\mathscr{I}_G=\bigcap_{Q\in \mathscr{P}(G)} \mathscr{I}(Q)$ be all independence and conditional independence common to all the distributions in $\mathscr{P}(G)$. If $\mathscr{I}(P)=\mathscr{I}_G$,  $P$ is called ``faithful'' to $G$. Otherwise, $P$ is called ``unfaithful'' to $G$.

In this paper, only $m=3$,$n=2$, and $\mathcal{X}_i=\{0,1\}$ are discussed, namely, binary three vertices DAG. For convenience, ${X}_1,{X}_2,{X}_3$ are denoted as $X,Y,Z$ respectively.

For any three random variables $(X,Y,Z)$, there are two important lemmas (\citet{Dawid:1979},
\citet{Pearl.etal:2022} , \citet{Lauritzen:1996} , \citet{Cowell.etal:1999} ) as follows.

\begin{lemma} For random variables $(X,Y,Z)$, one has\\
(1) (Symmetry) $X\perp\!\!\!\perp Y \Rightarrow Y\perp\!\!\!\perp X$,  \quad $X\perp\!\!\!\perp Y |Z \Rightarrow Y\perp\!\!\!\perp X|Z$.\\
(2) (Decomposition) $X\perp\!\!\!\perp (Y,Z) $ $\Rightarrow$  ($X\perp\!\!\!\perp Y$) and ($X\perp\!\!\!\perp Z$).\\
(3) (Weak union) $X\perp\!\!\!\perp (Y,Z) $ $\Rightarrow$  ($X\perp\!\!\!\perp Y | Z$) and ($X\perp\!\!\!\perp Z|Y $).\\
(4) (Contraction) ($X\perp\!\!\!\perp Y |Z$) and ($X\perp\!\!\!\perp Z $) $\Rightarrow$  $X\perp\!\!\!\perp (Y,Z)$ .
\end{lemma}

\begin{lemma} If  $X\perp\!\!\!\perp Y |Z$ and $X\perp\!\!\!\perp Z |Y$, then $X\perp\!\!\!\perp (Y ,Z)$.\\
lemma 2 holds under additionally conditions--essentially that there be no non--trivial logical relationship between Y and Z. OR, $(X,Y,Z)$ has a continuous positive probability density.
\end{lemma}

\begin{figure}[!htbp]
\begin{minipage}{\textwidth}
\centering  {Subset A}\\
\subfloat[{(i)}]{\includegraphics[width=2.5cm]{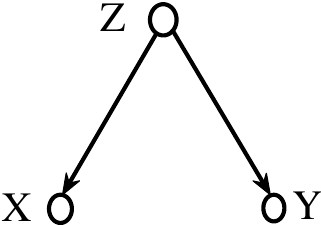}}
\hfill
\subfloat[{(ii)}]{\includegraphics[width=2.5cm]{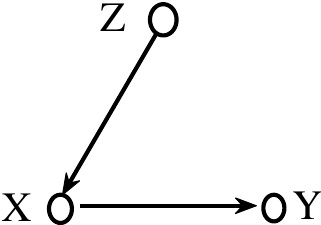}}
\hfill
\subfloat[{(iii)}]{\includegraphics[width=2.5cm]{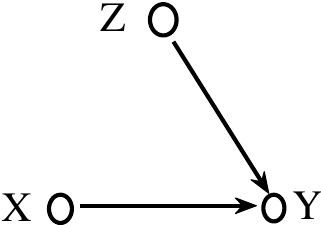}}
\hfill
\subfloat[{(iv)}]{\includegraphics[width=2.5cm]{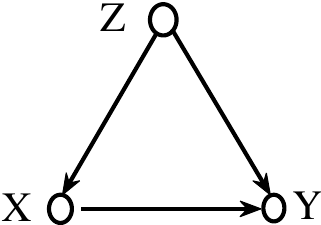}}
\hfill
\subfloat[{(v)}]{\includegraphics[width=2.5cm]{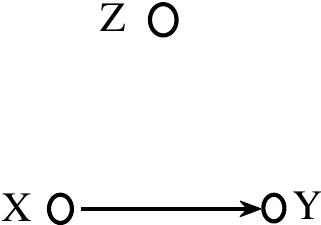}}
\end{minipage}
\qquad
\vfill
\vspace{0.2cm}
\begin{minipage}{0.6\textwidth}
\centering  {Subset B}\\
\subfloat[{(vi)}]{\includegraphics[width=2.5cm]{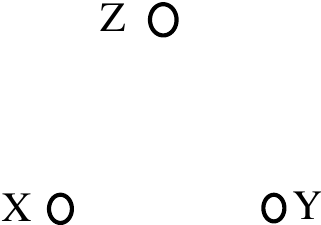}}
\hfill
\subfloat[{(vii)}]{\includegraphics[width=2.5cm]{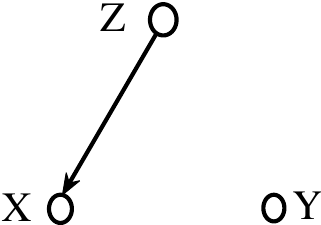}}
\hfill
\subfloat[{(viii)}]{\includegraphics[width=2.5cm]{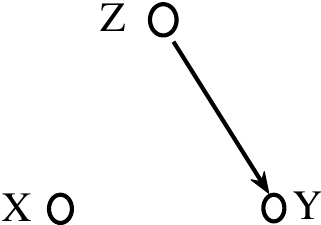}}

\end{minipage}
\caption{3-vertices causal DAGs from $X$ to $Y$ with confounder $Z$.}
\end{figure}

\begin{figure}[!htbp]
\begin{minipage}{0.6\textwidth}
\centering  {Subset C}\\
\subfloat[{(ix)}]{\includegraphics[width=2.5cm]{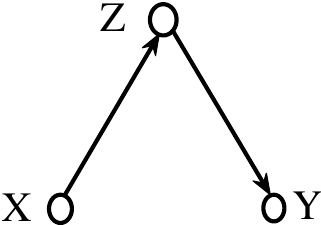}}
\hfill
\subfloat[{(x)}]{\includegraphics[width=2.5cm]{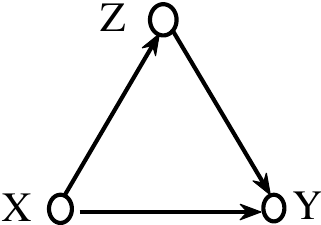}}
\hfill
\subfloat[{(xi)}]{\includegraphics[width=2.5cm]{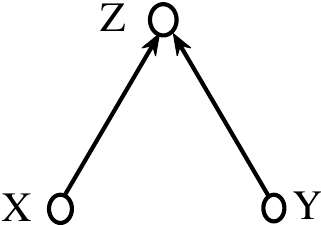}}
\end{minipage}
\caption{3-vertices causal DAGs from $X$ to $Y$, and $Z$ is not a confounder.}
\end{figure}

All possible probability independence and conditional independence of random variables $\{X,Y,Z\}$ are listed as follows, though there may be some relationships deduced by Lemma 1 and Lemma 2.
\begin{equation}
\begin{array}{ll}
\mathscr{I}_{\bigcup}(P)= & \left\{\emptyset; X\perp\!\!\!\perp Y; X\perp\!\!\!\perp Z; Y\perp\!\!\!\perp Z;
                     X\perp\!\!\!\perp (Y,Z); Y\perp\!\!\!\perp (X,Z); Z\perp\!\!\!\perp (X,Y); \right. \\
                   &\left.   X\perp\!\!\!\perp Y|Z; X\perp\!\!\!\perp Z|Y; Y\perp\!\!\!\perp Z|X; X\perp\!\!\!\perp Y\perp\!\!\!\perp Z \right\}
\end{array}
\end{equation}
where $X\perp\!\!\!\perp Y\perp\!\!\!\perp Z $ denotes mutual independence of $\{X,Y,Z\}$. The ``$;$'' in $\mathscr{I}_{\bigcup}(P)$ denotes union of all possible probability independence and conditional independence.

As the basic principle of faithfulness in causal DAG, the faithfulness of graph and the probability distribution satisfying intersection of all possible non--empty probability independence and conditional independence in $\mathscr{I}_{\bigcup}(P)$ is described in Theorem 1.

\begin{theorem}[Theorem 1]
The faithful graph of binary random variables $V=\{X,Y,Z\}$ satisfying
\begin{equation}
\begin{array}{ll}
\mathscr{I}_{\bigcap}(P)= & \left\{ X\perp\!\!\!\perp Y, X\perp\!\!\!\perp Z, Y\perp\!\!\!\perp Z, X\perp\!\!\!\perp (Y,Z), Y\perp\!\!\!\perp (X,Z), Z\perp\!\!\!\perp (X,Y), \right. \\
&\left. X\perp\!\!\!\perp Y|Z, X\perp\!\!\!\perp Z|Y, Y\perp\!\!\!\perp Z|X, X\perp\!\!\!\perp Y\perp\!\!\!\perp Z   \right\}

\end{array}
\end{equation}
is $G=(V,\emptyset)$  (Fig 1 (vi)). \\
where $\mathscr{I}_{\bigcap}(P)$ denotes the intersection of all non--empty probability independence and conditional independence.
And the ``$,$'' in $\mathscr{I}_{\bigcap}(P)$ denotes intersection of probability independence and conditional independence.
\end{theorem}

Theorem 1 implies strong assumptions that $\{P(X)>0,P(Y)>0,P(Z)>0,P(XY)>0,P(YZ)>0,P(XZ)>0\}$ and $P(XYZ)>0$.\footnote{The probability formula is written in abbreviate form in this paper.}
An probability distribution example of Theorem 1 is constructed as follows,

\begin{example}
Let the joint probability distribution of random variables $(X,Y,Z)$ be as follows,
\begin{equation}
P(X=i,Y=j,Z=k)=\frac{1}{8},  \forall  i,j,k=0,1.
\end{equation}
Then, $P(X,Y,Z)>0$ and satisfies
\begin{equation}
\begin{array}{ll}
\mathscr{I}(P)= & \left\{ X\perp\!\!\!\perp Y, X\perp\!\!\!\perp Z, Y\perp\!\!\!\perp Z,
                     X\perp\!\!\!\perp (Y,Z), Y\perp\!\!\!\perp (X,Z), Z\perp\!\!\!\perp (X,Y),\right. \\
                 &\left. X\perp\!\!\!\perp Y|Z, X\perp\!\!\!\perp Z|Y, Y\perp\!\!\!\perp Z|X,X\perp\!\!\!\perp Y\perp\!\!\!\perp Z   \right\}
\end{array}
\end{equation}
\end{example}

Example 1 is faithful to Fig 1 (vi), and the most important is a probability distribution with only $X\perp\!\!\!\perp Y\perp\!\!\!\perp Z$ may not be faithful to Fig 1 (vi) (see Example 7).

Theorem 1 establishes the basic principle of faithfulness structure of three--vertices causal DAG.
The basic triple causal structure will strictly determine causal discovery structure.

The relationship among pairwise independence, condition independence and mutual independence is obtained in Theorem 2.

\begin{theorem}[Theorem 2]
For three binary random variables $\{X,Y,Z\}$ , if $P(Z)>0$, then
\begin{equation}
\begin{array}{l}
\left\{ X\perp\!\!\!\perp Y, X\perp\!\!\!\perp Z, Y\perp\!\!\!\perp Z, X\perp\!\!\!\perp Y|Z \right\}  \Longleftrightarrow X\perp\!\!\!\perp Y\perp\!\!\!\perp Z.
\end{array}
\end{equation}
\end{theorem}

\begin{proof}
`'$\Longrightarrow$ ''
  Because $X\perp\!\!\!\perp Z$, $Y\perp\!\!\!\perp Z$ ,$ X\perp\!\!\!\perp Y|Z$, for $\forall i,j,k=0,1$, one has $P(X=i|Z=k)=P(X=i)$, $P(Y=j|Z=k)=P(Y=j)$, $P(X=i,Y=j|Z=k)=P(X=i|Z=k)P(Y=j|Z=k)$.
  Hence,$ P(X=i,Y=j,Z=k)=P(Z=k)P(X=i,Y=j|Z=k)=P(Z=k)P(X=i|Z=k)P(Y=j|Z=k)$
  $=P(X=i)P(Y=j)P(Z=k)$. Then, $X\perp\!\!\!\perp Y\perp\!\!\!\perp Z$ holds.\\
`'$\Longleftarrow$  ''
  Since $X\perp\!\!\!\perp Y\perp\!\!\!\perp Z$, for $\forall i,j,k=0,1$, one has 
  $X\perp\!\!\!\perp Z$, $Y\perp\!\!\!\perp Z$, $P(X=i|Z=k)=P(X=i)$, $P(Y=j|Z=k)=P(Y=j)$.
   And,$ P(X=i,Y=j|Z=k)=P(X=i,Y=j,Z=k)/P(Z=k)$\\
   $=P(X=i)P(Y=j)P(Z=k)/P(Z=k)=P(X=i)P(Y=j)=P(X=i|Z=k)P(Y=j|Z=k)$. 
   Hence, $X\perp\!\!\!\perp Y$, $X\perp\!\!\!\perp Z$, $Y\perp\!\!\!\perp Z$, $X\perp\!\!\!\perp Y|Z$ hold.

\end{proof}

\section{Unfaithful probability distribution in binary three vertices DAG} \label{sec3}

Several typical unfaithful probability distributions are constructed in this section, the verification of unfaithfulness is supported by Lemma 1, Lemma 2, Theorem 1 and d--separation \footnote{The detail calculation of verification is omitted for paper length limit. And, one can design as many as possible counterexamples similar to the given structures using mathematical skills. }.
\begin{example}
Let the joint probability distribution of random variables $(X,Y,Z)$ be as follows,
\begin{equation}
\begin{array}{lll}
P(X=0,Y=0,Z=0)=\frac{1}{16}, & P(X=0,Y=0,Z=1)=\frac{3}{16}\\
P(X=0,Y=1,Z=0)=\frac{3}{16}, & P(X=0,Y=1,Z=1)=\frac{1}{16}\\
P(X=1,Y=0,Z=0)=\frac{3}{16}, & P(X=1,Y=0,Z=1)=\frac{1}{16}\\
P(X=1,Y=1,Z=0)=\frac{1}{16}, & P(X=1,Y=1,Z=1)=\frac{3}{16}\\
\end{array}
\end{equation}
Then, $P(X,Y,Z)$ ($P(X,Y,Z)>0$) satisfies
\begin{equation}
\begin{array}{ll}
\mathscr{I}(P)= & \left\{ X\perp\!\!\!\perp Y, Y\perp\!\!\!\perp Z, X\perp\!\!\!\perp Z \right\}
\end{array}
\end{equation}
\end{example}

According to Theorem 1, Example 1 is faithful to Fig 1 (vi) and it has more intersection independence and conditional independence information than Example 2. Though the most possible faithful graph of Example 2  is Fig 1 (vi),  Example 2 is unfaithful to Fig 1 (vi).
\begin{example}
Let the joint probability distribution of random variables $(X,Y,Z)$ be as follows,
\begin{equation}
\begin{array}{lll}
P(X=0,Y=0,Z=0)=\frac{1}{4}, & P(X=0,Y=0,Z=1)=\frac{4}{15}\\
P(X=0,Y=1,Z=0)=\frac{1}{12}, & P(X=0,Y=1,Z=1)=\frac{1}{15}\\
P(X=1,Y=0,Z=0)=\frac{1}{8}, & P(X=1,Y=0,Z=1)=\frac{4}{30}\\
P(X=1,Y=1,Z=0)=\frac{1}{24}, & P(X=1,Y=1,Z=1)=\frac{1}{30}\\
\end{array}
\end{equation}
Then, $P(X,Y,Z)$ ($P(X,Y,Z)>0$) satisfies
\begin{equation}
\begin{array}{ll}
\mathscr{I}(P)= & \left\{ X\perp\!\!\!\perp Y, X\perp\!\!\!\perp Z, X\perp\!\!\!\perp (Y,Z),
                      X\perp\!\!\!\perp Y|Z \right\}
\end{array}
\end{equation}
\end{example}

According to Lemma 1 (4)(2), $X\perp\!\!\!\perp Y|Z, X\perp\!\!\!\perp Z$ $\Rightarrow$ $X\perp\!\!\!\perp (Y,Z)$ $\Rightarrow$ $Y \perp\!\!\!\perp Z$. Applying  Theorem 1, this probability distribution is unfaithful to Fig 1 (vi), and also not faithful to any DAG in Fig 1 and Fig 2. This example is similar to one in \citet{Ramsey.etal:2006}, which goal is to show how the PC algorithm (\citet{Spirtes.etal:2000}) errs.

\begin{example}
Let the joint probability distribution of random variables $(X,Y,Z)$ be as follows,
\begin{equation}
\begin{array}{lll}
P(X=0,Y=0,Z=0)=\frac{1}{5}, & P(X=0,Y=0,Z=1)=0\\
P(X=0,Y=1,Z=0)=0, & P(X=0,Y=1,Z=1)=\frac{3}{10}\\
P(X=1,Y=0,Z=0)=\frac{3}{10}, & P(X=1,Y=0,Z=1)=0\\
P(X=1,Y=1,Z=0)=0, & P(X=1,Y=1,Z=1)=\frac{1}{5}\\
\end{array}
\end{equation}
Then, $P(X,Y,Z)$ satisfies
\begin{equation}
\begin{array}{ll}
\mathscr{I}(P)= & \left\{ X\perp\!\!\!\perp Y|Z, X\perp\!\!\!\perp Z|Y \right\}
\end{array}
\end{equation}
\end{example}
And, $P(X|Y,Z)$ does not exist.
Example 4 is  unfaithful to any DAG of Fig 1 and Fig 2. It shows that generally $ X\perp\!\!\!\perp Y|Z, X\perp\!\!\!\perp Z|Y$ can not deduce $ X\perp\!\!\!\perp (Y,Z)$. This is also a counterexample of Lemma 2.

\begin{example}
Let the joint probability distribution of random variables $(X,Y,Z)$ be as follows,
\begin{equation}
\begin{array}{lll}
P(X=0,Y=0,Z=0)=\frac{1}{4}, & P(X=0,Y=0,Z=1)=\frac{3}{10}\\
P(X=0,Y=1,Z=0)=\frac{1}{8}, & P(X=0,Y=1,Z=1)=\frac{3}{40}\\
P(X=1,Y=0,Z=0)=\frac{1}{12}, & P(X=1,Y=0,Z=1)=\frac{1}{10}\\
P(X=1,Y=1,Z=0)=\frac{1}{24}, & P(X=1,Y=1,Z=1)=\frac{1}{40}\\
\end{array}
\end{equation}
Then, $P(X,Y,Z)$ satisfies
\begin{equation}
\begin{array}{ll}
\mathscr{I}(P)= & \left\{ X\perp\!\!\!\perp Y, X\perp\!\!\!\perp Z \right\}
\end{array}
\end{equation}
\end{example}

Example 5 satisfies $\left\{ X\perp\!\!\!\perp Y, X\perp\!\!\!\perp Z \right\}$, however it is not faithful to any graph of Fig 1 and Fig 2.
As Fig 1 (viii) implies $\left\{ X\perp\!\!\!\perp (Y,Z) \right\}$, which does not hold in Example 5. Example 5 is not faithful to Fig 1 (viii).

\begin{example}
Let the joint probability distribution of random variables $(X,Y,Z)$ be as follows,
\begin{equation}
\begin{array}{lll}
P(X=0,Y=0,Z=0)=\frac{1}{4}, & P(X=0,Y=0,Z=1)=0\\
P(X=0,Y=1,Z=0)=0, & P(X=0,Y=1,Z=1)=0\\
P(X=1,Y=0,Z=0)=0, & P(X=1,Y=0,Z=1)=0\\
P(X=1,Y=1,Z=0)=0, & P(X=1,Y=1,Z=1)=\frac{3}{4}\\
\end{array}
\end{equation}
Then, $P(X,Y,Z)$ satisfies
\begin{equation}
\begin{array}{ll}
\mathscr{I}(P)= & \left\{ X \perp\!\!\!\perp Y|Z, X\perp\!\!\!\perp Z|Y, Y\perp\!\!\!\perp Z|X \right\}
\end{array}
\end{equation}
\end{example}

Example 6 is not faithful to any DAG of Fig 1 and Fig 2 according to PC algorithm. It is also a counterexample of Lemma 2.
Moreover, this example also shows that
$\left\{ X \perp\!\!\!\perp Y|Z, X\perp\!\!\!\perp Z|Y, Y\perp\!\!\!\perp Z|X \right\}\nRightarrow \left\{ X \perp\!\!\!\perp Y, X\perp\!\!\!\perp Z, Y\perp\!\!\!\perp Z \right\} $, however it is true  in Gaussian distribution case \citet{Sullivant:2009}.

\begin{example}
Let the joint probability distribution of random variables $(X,Y,Z)$ be as follows,
\begin{equation}
\begin{array}{lll}
P(X=0,Y=0,Z=0)=\frac{1}{4}, & P(X=0,Y=0,Z=1)=\frac{3}{4}\\
P(X=0,Y=1,Z=0)=0, & P(X=0,Y=1,Z=1)=0\\
P(X=1,Y=0,Z=0)=0, & P(X=1,Y=0,Z=1)=0\\
P(X=1,Y=1,Z=0)=0, & P(X=1,Y=1,Z=1)=0\\
\end{array}
\end{equation}
Then, $P(X,Y,Z)$ satisfies
\begin{equation}
\begin{array}{ll}
\mathscr{I}(P)= & \left\{X\perp\!\!\!\perp Y, X\perp\!\!\!\perp Z, Y\perp\!\!\!\perp Z, X \perp\!\!\!\perp Y|Z \right\}
\end{array}
\end{equation}
\end{example}

 Though $\left\{X\perp\!\!\!\perp Y, X\perp\!\!\!\perp Z, Y\perp\!\!\!\perp Z, X \perp\!\!\!\perp Y|Z \right\} $ $\Longrightarrow$ $\left\{ X \perp\!\!\!\perp Y \perp\!\!\!\perp Z \right\} $ according to Theorem 2, Example 7 is not faithful to Fig 1 (vi)  according to Theorem 1. And, it is unfaithful to any DAG in  Fig 1 and Fig 2.

\begin{example}
Let the joint probability distribution of random variables $(X,Y,Z)$ be as follows,
\begin{equation}
\begin{array}{lll}
P(X=0,Y=0,Z=0)=\frac{1}{3}, & P(X=0,Y=0,Z=1)=0\\
P(X=0,Y=1,Z=0)=0,           & P(X=0,Y=1,Z=1)=0\\
P(X=1,Y=0,Z=0)=\frac{1}{3}, & P(X=1,Y=0,Z=1)=0\\
P(X=1,Y=1,Z=0)=\frac{1}{3}, & P(X=1,Y=1,Z=1)=0\\
\end{array}
\end{equation}
Then, $P(X,Y,Z)$ satisfies
\begin{equation}
\begin{array}{ll}
\mathscr{I}(P)= & \left\{ Y\perp\!\!\!\perp Z, X\perp\!\!\!\perp Z, Y\perp\!\!\!\perp Z|X, X \perp\!\!\!\perp Z|Y \right\}
\end{array}
\end{equation}
\end{example}

Example 8 is not faithful to any graph of Fig 1 and Fig 2, since $ Z\perp\!\!\!\perp (X, Y)$ does not hold.

\begin{example}
Let the joint probability distribution of random variables $(X,Y,Z)$ be as follows,
\begin{equation}
\begin{array}{lll}
P(X=0,Y=0,Z=0)=0,           & P(X=0,Y=0,Z=1)=0\\
P(X=0,Y=1,Z=0)=0,           & P(X=0,Y=1,Z=1)=0\\
P(X=1,Y=0,Z=0)=\frac{1}{4}, & P(X=1,Y=0,Z=1)=\frac{1}{4}\\
P(X=1,Y=1,Z=0)=\frac{1}{4}, & P(X=1,Y=1,Z=1)=\frac{1}{4}\\
\end{array}
\end{equation}
Then, $P(X,Y,Z)$ satisfies
\begin{equation}
\begin{array}{ll}
\mathscr{I}(P)= & \left\{  X \perp\!\!\!\perp Y, Y\perp\!\!\!\perp Z, X\perp\!\!\!\perp Z, X\perp\!\!\!\perp Y|Z, X\perp\!\!\!\perp Z|Y, X\perp\!\!\!\perp (Y,Z) \right\}
\end{array}
\end{equation}
\end{example}

Example 9 is not faithful to any graph of Fig 1 and Fig 2. According to Theorem 2, $\left\{  X \perp\!\!\!\perp Y, Y\perp\!\!\!\perp Z, X\perp\!\!\!\perp Z, X\perp\!\!\!\perp Y|Z, X\perp\!\!\!\perp Z|Y, X\perp\!\!\!\perp (Y,Z) \right\}$ $\Longrightarrow$ $X\perp\!\!\!\perp Y\perp\!\!\!\perp Z$, but it is not faithful to Fig 1 (vi) according to Theorem 1.

\begin{example}
Let the joint probability distribution of random variables $(X,Y,Z)$ be as follows,
\begin{equation}
\begin{array}{lll}
P(X=0,Y=0,Z=0)=1,           & P(X=0,Y=0,Z=1)=0\\
P(X=0,Y=1,Z=0)=0,           & P(X=0,Y=1,Z=1)=0\\
P(X=1,Y=0,Z=0)=0, & P(X=1,Y=0,Z=1)=0\\
P(X=1,Y=1,Z=0)=0, & P(X=1,Y=1,Z=1)=0\\
\end{array}
\end{equation}
Then, $P(X,Y,Z)$ satisfies
\begin{equation}
\begin{array}{ll}
\mathscr{I}(P)= & \left\{  X \perp\!\!\!\perp Y, Y\perp\!\!\!\perp Z, X\perp\!\!\!\perp Z, X\perp\!\!\!\perp Y \perp\!\!\!\perp Z \right\}
\end{array}
\end{equation}
\end{example}

Although Example 10 satisfies pairwise independence and mutual independence, none conditional independence exists, it is not faithful to Fig 1 (vi) according to Theorem 1. And, it is unfaithful to any DAG of Fig 1 and Fig 2.

\section{General unfaithful probability distribution in Triple causal DAG } \label{sec4}

Generally, the faithfulness of probability distribution in binary triple $(X,Y,Z)$ with  `'  Cause X and Outcome Y'' is obtained in Theorem 3.

\begin{theorem}[Theorem 3]
For any binary triple $(X,Y,Z)$ within  `'  Cause X and Outcome Y'' causal structure, consider all the intersection of multiple combination of independence and conditional independence as follows
\begin{equation}
\begin{array}{ll}
\mathscr{I}_{I:1}(P)= & \left\{X\perp\!\!\!\perp Y, X\perp\!\!\!\perp Z \right\} \\
\mathscr{I}_{I:2}(P)= & \left\{X\perp\!\!\!\perp Y, Y\perp\!\!\!\perp Z \right\} \\
\mathscr{I}_{I:3}(P)= & \left\{X\perp\!\!\!\perp Z, Y\perp\!\!\!\perp Z \right\} \\
\mathscr{I}_{I:4}(P)= & \left\{X\perp\!\!\!\perp Y, X\perp\!\!\!\perp Z, Y\perp\!\!\!\perp Z \right\} \\
\mathscr{I}_{I:5}(P)= & \left\{X \perp\!\!\!\perp Y|Z, X \perp\!\!\!\perp Z|Y \right\} \\
\mathscr{I}_{I:6}(P)= & \left\{X \perp\!\!\!\perp Y|Z, Y \perp\!\!\!\perp Z|X \right\} \\
\mathscr{I}_{I:7}(P)= & \left\{Y \perp\!\!\!\perp Z|X, X \perp\!\!\!\perp Z|Y \right\} \\
\mathscr{I}_{I:8}(P)= & \left\{X \perp\!\!\!\perp Y|Z, X \perp\!\!\!\perp Z|Y, Y\perp\!\!\!\perp Z|X \right\} \\
\mathscr{I}_{II:1}(P)= & \left\{X\perp\!\!\!\perp Y, X \perp\!\!\!\perp Y|Z \right\} \\
\mathscr{I}_{II:2}(P)= & \left\{X\perp\!\!\!\perp Y, X \perp\!\!\!\perp Z|Y \right\} \\
\mathscr{I}_{II:3}(P)= & \left\{X\perp\!\!\!\perp Y, Y \perp\!\!\!\perp Z|X \right\} \\
\mathscr{I}_{II:4}(P)= & \left\{X\perp\!\!\!\perp Z, X \perp\!\!\!\perp Y|Z \right\} \\
\mathscr{I}_{II:5}(P)= & \left\{X\perp\!\!\!\perp Z, X \perp\!\!\!\perp Z|Y \right\} \\
\mathscr{I}_{II:6}(P)= & \left\{X\perp\!\!\!\perp Z, Y \perp\!\!\!\perp Z|X \right\} \\
\mathscr{I}_{II:7}(P)= & \left\{Y\perp\!\!\!\perp Z, X \perp\!\!\!\perp Y|Z \right\} \\
\mathscr{I}_{II:8}(P)= & \left\{Y\perp\!\!\!\perp Z, X \perp\!\!\!\perp Z|Y \right\} \\
\mathscr{I}_{II:9}(P)= & \left\{Y\perp\!\!\!\perp Z, Y \perp\!\!\!\perp Z|X \right\} \\
\end{array}
\end{equation}

\begin{equation}
\begin{array}{ll}
\mathscr{I}_{III:1}(P)= & \left\{X\perp\!\!\!\perp Y, X \perp\!\!\!\perp Y|Z, X \perp\!\!\!\perp Z|Y \right\} \\
\mathscr{I}_{III:2}(P)= & \left\{X\perp\!\!\!\perp Y, X \perp\!\!\!\perp Z|Y, Y \perp\!\!\!\perp Z|X \right\} \\
\mathscr{I}_{III:3}(P)= & \left\{X\perp\!\!\!\perp Y, Y \perp\!\!\!\perp Z|X, X \perp\!\!\!\perp Y|Z \right\} \\
\mathscr{I}_{III:4}(P)= & \left\{X\perp\!\!\!\perp Z, X \perp\!\!\!\perp Y|Z, X \perp\!\!\!\perp Z|Y \right\} \\
\mathscr{I}_{III:5}(P)= & \left\{X\perp\!\!\!\perp Z, X \perp\!\!\!\perp Z|Y, Y \perp\!\!\!\perp Z|X \right\} \\
\mathscr{I}_{III:6}(P)= & \left\{X\perp\!\!\!\perp Z, Y \perp\!\!\!\perp Z|X, X \perp\!\!\!\perp Y|Z \right\} \\
\mathscr{I}_{III:7}(P)= & \left\{Y\perp\!\!\!\perp Z, X \perp\!\!\!\perp Y|Z, X \perp\!\!\!\perp Z|Y \right\} \\
\mathscr{I}_{III:8}(P)= & \left\{Y\perp\!\!\!\perp Z, X \perp\!\!\!\perp Z|Y, Y \perp\!\!\!\perp Z|X \right\} \\
\mathscr{I}_{III:9}(P)= & \left\{Y\perp\!\!\!\perp Z, Y \perp\!\!\!\perp Z|X, X \perp\!\!\!\perp Y|Z \right\} \\
\end{array}
\end{equation}


\begin{equation}
\begin{array}{ll}
\mathscr{I}_{IV:1}(P)= & \left\{X\perp\!\!\!\perp Y, X \perp\!\!\!\perp Y|Z, X \perp\!\!\!\perp Z|Y, Y \perp\!\!\!\perp Z|X \right\} \\
\mathscr{I}_{IV:2}(P)= & \left\{X\perp\!\!\!\perp Z, X \perp\!\!\!\perp Y|Z, X \perp\!\!\!\perp Z|Y, Y \perp\!\!\!\perp Z|X \right\} \\
\mathscr{I}_{IV:3}(P)= & \left\{Y\perp\!\!\!\perp Z, X \perp\!\!\!\perp Y|Z, X \perp\!\!\!\perp Z|Y, Y \perp\!\!\!\perp Z|X \right\} \\
\end{array}
\end{equation}

\begin{equation}
\begin{array}{ll}
\mathscr{I}_{V:1}(P)= & \left\{X\perp\!\!\!\perp Y, X\perp\!\!\!\perp Z, X \perp\!\!\!\perp Y|Z \right\} \\
\mathscr{I}_{V:2}(P)= & \left\{X\perp\!\!\!\perp Y, X\perp\!\!\!\perp Z, X \perp\!\!\!\perp Z|Y \right\} \\
\mathscr{I}_{V:3}(P)= & \left\{X\perp\!\!\!\perp Y, X\perp\!\!\!\perp Z, Y \perp\!\!\!\perp Z|X \right\} \\
\mathscr{I}_{V:4}(P)= & \left\{X\perp\!\!\!\perp Z, Y\perp\!\!\!\perp Z, X \perp\!\!\!\perp Y|Z \right\} \\
\mathscr{I}_{V:5}(P)= & \left\{X\perp\!\!\!\perp Z, Y\perp\!\!\!\perp Z, X \perp\!\!\!\perp Z|Y \right\} \\
\mathscr{I}_{V:6}(P)= & \left\{X\perp\!\!\!\perp Z, Y\perp\!\!\!\perp Z, Y \perp\!\!\!\perp Z|X \right\} \\
\mathscr{I}_{V:7}(P)= & \left\{Y\perp\!\!\!\perp Z, X\perp\!\!\!\perp Y, X \perp\!\!\!\perp Y|Z \right\} \\
\mathscr{I}_{V:8}(P)= & \left\{Y\perp\!\!\!\perp Z, X\perp\!\!\!\perp Y, X \perp\!\!\!\perp Z|Y \right\} \\
\mathscr{I}_{V:9}(P)= & \left\{Y\perp\!\!\!\perp Z, X\perp\!\!\!\perp Y, Y \perp\!\!\!\perp Z|X \right\} \\
\end{array}
\end{equation}

\begin{equation}
\begin{array}{ll}
\mathscr{I}_{VI:1}(P)= & \left\{X\perp\!\!\!\perp Y, X\perp\!\!\!\perp Z, X \perp\!\!\!\perp Y|Z, X \perp\!\!\!\perp Z|Y \right\} \\
\mathscr{I}_{VI:2}(P)= & \left\{X\perp\!\!\!\perp Y, X\perp\!\!\!\perp Z, X \perp\!\!\!\perp Z|Y, Y \perp\!\!\!\perp Z|X \right\} \\
\mathscr{I}_{VI:3}(P)= & \left\{X\perp\!\!\!\perp Y, X\perp\!\!\!\perp Z, Y \perp\!\!\!\perp Z|X, X \perp\!\!\!\perp Y|Z \right\} \\
\mathscr{I}_{VI:4}(P)= & \left\{X\perp\!\!\!\perp Z, Y\perp\!\!\!\perp Z, X \perp\!\!\!\perp Y|Z, X \perp\!\!\!\perp Z|Y \right\} \\
\mathscr{I}_{VI:5}(P)= & \left\{X\perp\!\!\!\perp Z, Y\perp\!\!\!\perp Z, X \perp\!\!\!\perp Z|Y, Y \perp\!\!\!\perp Z|X \right\} \\
\mathscr{I}_{VI:6}(P)= & \left\{X\perp\!\!\!\perp Z, Y\perp\!\!\!\perp Z, Y \perp\!\!\!\perp Z|X, X \perp\!\!\!\perp Y|Z \right\} \\
\mathscr{I}_{VI:7}(P)= & \left\{Y\perp\!\!\!\perp Z, X\perp\!\!\!\perp Y, X \perp\!\!\!\perp Y|Z, X \perp\!\!\!\perp Z|Y \right\} \\
\mathscr{I}_{VI:8}(P)= & \left\{Y\perp\!\!\!\perp Z, X\perp\!\!\!\perp Y, X \perp\!\!\!\perp Z|Y, Y \perp\!\!\!\perp Z|X \right\} \\
\mathscr{I}_{VI:9}(P)= & \left\{Y\perp\!\!\!\perp Z, X\perp\!\!\!\perp Y, Y \perp\!\!\!\perp Z|X, X \perp\!\!\!\perp Y|Z \right\} \\
\end{array}
\end{equation}

\begin{equation}
\begin{array}{ll}
\mathscr{I}_{VII:1}(P)= & \left\{X\perp\!\!\!\perp Y, X\perp\!\!\!\perp Z, X \perp\!\!\!\perp Y|Z, X \perp\!\!\!\perp Z|Y, Y \perp\!\!\!\perp Z|X \right\} \\
\mathscr{I}_{VII:2}(P)= & \left\{X\perp\!\!\!\perp Z, Y\perp\!\!\!\perp Z, X \perp\!\!\!\perp Y|Z, X \perp\!\!\!\perp Z|Y, Y \perp\!\!\!\perp Z|X \right\} \\
\mathscr{I}_{VII:3}(P)= & \left\{Y\perp\!\!\!\perp Z, X\perp\!\!\!\perp Y, X \perp\!\!\!\perp Y|Z, X \perp\!\!\!\perp Z|Y, Y \perp\!\!\!\perp Z|X \right\} \\
\end{array}
\end{equation}
\begin{equation}
\begin{array}{ll}
\mathscr{I}_{VIII:1}(P)= & \left\{X\perp\!\!\!\perp Y, X\perp\!\!\!\perp Z, Y\perp\!\!\!\perp Z, X \perp\!\!\!\perp Y|Z \right\} \\
\mathscr{I}_{VIII:2}(P)= & \left\{X\perp\!\!\!\perp Y, X\perp\!\!\!\perp Z, Y\perp\!\!\!\perp Z, X \perp\!\!\!\perp Z|Y \right\} \\
\mathscr{I}_{VIII:3}(P)= & \left\{X\perp\!\!\!\perp Y, X\perp\!\!\!\perp Z, Y\perp\!\!\!\perp Z, Y \perp\!\!\!\perp Z|X \right\} \\

\mathscr{I}_{IX:1}(P)= & \left\{X\perp\!\!\!\perp Y, X\perp\!\!\!\perp Z, Y\perp\!\!\!\perp Z, X \perp\!\!\!\perp Y|Z, X \perp\!\!\!\perp Z|Y \right\} \\
\mathscr{I}_{IX:2}(P)= & \left\{X\perp\!\!\!\perp Y, X\perp\!\!\!\perp Z, Y\perp\!\!\!\perp Z, X \perp\!\!\!\perp Z|Y, Y \perp\!\!\!\perp Z|X \right\} \\
\mathscr{I}_{IX:3}(P)= & \left\{X\perp\!\!\!\perp Y, X\perp\!\!\!\perp Z, Y\perp\!\!\!\perp Z, Y \perp\!\!\!\perp Z|X, X \perp\!\!\!\perp Y|Z \right\} \\
\mathscr{I}_{X:1}(P)= & \left\{X\perp\!\!\!\perp Y, X\perp\!\!\!\perp Z, Y\perp\!\!\!\perp Z, X \perp\!\!\!\perp Y|Z, X \perp\!\!\!\perp Z|Y, Y \perp\!\!\!\perp Z|X \right\} \\
\end{array}
\end{equation}

Then, one has\\
(1) All probability distributions from $\mathscr{I}_{I:1}(P)$ to $\mathscr{I}_{I:9}(P) $ are not faithful to any graph of Fig 1 and Fig 2.\\
(2) $\mathscr{I}_{II:1}(P)$ is faithful to Fig 2 (vi), $\mathscr{I}_{II:9}(P)$ is faithful to Fig 1 (ii), probability distributions from $\mathscr{I}_{II:2}(P)$ to $\mathscr{I}_{II:8}(P) $ are not faithful to any graph of Fig 1 and Fig 2.\\
(3) $\mathscr{I}_{III:k}(P)\Rightarrow X\perp\!\!\!\perp Y \perp\!\!\!\perp Z $, $\forall k=1\cdots 9$. All of them are unfaithful to Fig 1 (vi). \\
    And, $\mathscr{I}_{III:k}(P)=\mathscr{I}_{VI:k}(P)=\mathscr{I}_{VII:l}(P) $, $\forall k=3m+l,m=0,1,2,l=1,2,3$.\\
(4) $\mathscr{I}_{IV:k}(P)= \mathscr{I}_{X:1}(P) $, $\forall k=1\cdots 3$. And, all of them are faithful to Fig 1 (vi). \\
(5) $\mathscr{I}_{V:k}(P)\Rightarrow X\perp\!\!\!\perp Y \perp\!\!\!\perp Z $, $\forall k=1\cdots 9$. All of them are unfaithful to Fig 1 (vi). \\
    And, $\mathscr{I}_{V:1}(P)=\mathscr{I}_{VIII:1}(P)$,  $\mathscr{I}_{V:2}(P)=\mathscr{I}_{VIII:2}(P)$,  $\mathscr{I}_{V:5}(P)=\mathscr{I}_{VIII:2}(P)$,\\
    $\mathscr{I}_{V:6}(P)=\mathscr{I}_{VIII:3}(P)$, $\mathscr{I}_{V:7}(P)=\mathscr{I}_{VIII:1}(P)$,
    $\mathscr{I}_{V:9}(P)=\mathscr{I}_{VIII:3}(P)$. \\
(6) $\mathscr{I}_{VI:k}(P)\Rightarrow X\perp\!\!\!\perp Y \perp\!\!\!\perp Z $, $\forall k=1\cdots 9$. All of them are unfaithful to Fig 1 (vi). \\
    And, $\mathscr{I}_{VI:k}(P)=\mathscr{I}_{IX:l}(P) $, $\forall k=3m+l,m=0,1,2,l=1,2,3$.\\
(7) $\mathscr{I}_{VII:k}(P)= \mathscr{I}_{X:1}(P) $, $\forall k=1\cdots 3$. And, all of them are faithful to Fig 1 (vi). \\
(8) $\mathscr{I}_{VIII:k}(P) \Rightarrow X\perp\!\!\!\perp Y \perp\!\!\!\perp Z$, $\forall k=1\cdots 3$. And, all of them are unfaithful to Fig 1 (vi). \\
(9) $\mathscr{I}_{IV:k}(P) \Rightarrow X\perp\!\!\!\perp Y \perp\!\!\!\perp Z$, $\forall k=1\cdots 3$. And, all of them are unfaithful to Fig 1 (vi). \\
(10) $\mathscr{I}_{X:1}(P)=\mathscr{I}_{\bigcap}(P) $. And, it is faithful to Fig 1 (vi). \\
\end{theorem}

\begin{proof}
Applying Lemma 1, Lemma 2, Theorem 1, Theorem 2 and the counterexamples, Theorem 3 can be proved and verified one--by--one. Here, only the proof of (10) is shown.\\
Since $ X \perp\!\!\!\perp Y|Z, X \perp\!\!\!\perp Z|Y, Y \perp\!\!\!\perp Z|X$ exist,  $P(X)>0$,$P(Y)>0$,$P(Z)>0$ hold. According to Theorem 2,
 $\mathscr{I}_{X:1}(P)\Rightarrow X\perp\!\!\!\perp Y \perp\!\!\!\perp Z$, hence $P(X,Y,Z)=P(X)P(Y)P(Z)>0$. Applying Lemma 1 and Theorem 1, we have $\mathscr{I}_{X:1}(P)=\mathscr{I}_{\bigcap}(P) $. Therefore, it is faithful to Fig 1 (vi).
\end{proof}

\section{Discussion} \label{sec5}

Triple DAG is the fundamental structure of causal discovery and causal statistical inference, if multiple independence and conditional independence happens simultaneously in the same probability distribution, faithfulness or unfaithfulness will take critical role to discover a causal DAG. The counterexamples and conclusions aforementioned are based on binary probability distribution, as there are infinite probability distribution families according to measure theory, and different probability distribution family has different property of independence and conditional independence, the research method in this paper provide a methodology and guideline for causal discovery with any given probability distribution family in  multiple independence and conditional independence case. 
The future work will focus on practical discrete probability, Gaussian normal distribution and other practical probability distributions. Furthermore, the most important is that unfaithfulness will lead to PC algorithm fail. The obtained examples and conclusions not only enrich the Spirtes--Glymour--Scheines causal model, but also will be helpful to modify PC algorithm in causal DAG structure refinement.

\section{Author contributions statement}

J.W. Liu conceived the whole theorem framework, counterexamples and theorems, wrote and reviewed the manuscript.

\section{Acknowledgments}

{\em Conflicts of interest}: None declared.


\bibliographystyle{abbrvnat}
\bibliography{reference}

\end{document}